\documentclass[letterpaper, 10 pt, conference]{ieeeconf}  

\IEEEoverridecommandlockouts                              
\usepackage[utf8]{inputenc}
\usepackage[T1]{fontenc}
\usepackage[printwatermark]{xwatermark}

\usepackage{amsmath}
\usepackage{amsfonts}
\usepackage{amssymb}
\usepackage{mathtools}
\usepackage{xcolor}
\usepackage{amsthm}
\newtheorem{theorem}{Theorem}

\newtheorem*{definition}{Definition}

\DeclareMathOperator*{\argmin}{arg\,min}
\usepackage{graphicx}
\usepackage{subfig}
\newcommand{\mathbold}[1]{\boldsymbol{#1}}
\usepackage{flushend}
\usepackage[font=small,skip=0pt]{caption}
\usepackage{caption}
\usepackage{float}
\title{\LARGE \bf
Differential Flatness and Flatness Inspired Control of Aerial Manipulators based on Lagrangian Reduction}

\author{Skylar X. Wei$^{1}$, Peder Hårderup$^{2}$, and Joel W. Burdick$^{1}$
\thanks{$^{1}$Engineering and Applied Sciences, California Institute of Technology, Pasadena, CA, USA.
      $\qquad$
\texttt{\{swei, jburdick\}@caltech.edu.}}%
\thanks{$^{2}$Division of Decision and Control Systems, KTH Royal Institute of Technology, Stockholm, Sweden. $\;$
        \texttt{\{pederhar\}@kth.se.}}%
\thanks{The work is partially funded by NASA Jet Propulsion Laboratory and California Institute of Technology Graduate Student Fellowship.}
        }

\begin{document}

\maketitle
\thispagestyle{empty}
\pagestyle{empty}

\begin{abstract}
This paper shows that the dynamics of a general class of aerial manipulators, consist of an underactuated multi-rotor base with an arbitrary k-linked articulated manipulator, are differentially flat. Methods of Lagrangian Reduction under broken symmetries produce reduced equations of motion whose key variables: center-of-mass linear momentum, vehicle yaw angle, and manipulator relative joint angles become the flat outputs.  Utilizing flatness theory and a second-order dynamic extension of the thrust input, we transform the mechanics of aerial manipulators to their equivalent trivial form with a valid relative degree. Using this flatness transformation, a quadratic programming based controller is proposed within a Control Lyapunov Function (CLF-QP) framework, and its performance is verified in simulation.
\end{abstract}

\section{INTRODUCTION}

An aerial manipulator consists of an underactuated flying multi-rotor body with a multi-link manipulator attached mainly, but not exclusively, at the body's geometric center.  Aerial manipulation is of increasing interest \cite{review} since such systems inherit the high mobility of conventional multi-rotor drones, with the added ability to interact with the environment via the robot arm's end-effector.  Aerial manipulation has many practical applications, such as delivering packages and payloads \cite{delivery}, inspection of physical infrastructure using arm-mounted sensors \cite{infrastructureinspection}\cite{bridgeinspection}, and tool operation \cite{tooloperation}.

Aerial manipulators present several challenges in trajectory planning and control.  First, they are typically underactuated.  Moreover, movements of their heavy arms or heavy payloads can cause complex shifts in the overall center-of-mass (CoM), which can potentially induce instabilities.


This paper shows that a general class of aerial manipulators is {\em differentially flat}. Flat systems are equivalent to a trivial system via an endogenous transformation \cite{Fliess}, which enables dynamic feedback linearization.  Hence, our results lead to nonlinear controllability results for these complex aircraft, and to the design of locally exponentially stabilizing controllers.  Multi-rotors (without robot arms) are known to be differentially flat, and important multi-rotor trajectory planning methods are based on this fact \cite{flatquad1}\cite{flatquad2}.  Our result allows methods developed for conventional multi-rotors to be generalized to aerial manipulators.

Prior efforts to prove the differential flatness of aerial manipulators have required assumptions that are limiting in practice. 
For example, \cite{flatAM_endeffector} showed that an aerial manipulator with a 2-DoF arm is differentially flat, but assumed that the CoM must be fixed in the end-effector frame, which unrealistically implies a massless or motionless arm.  This result was generalized in \cite{flatAM_decoupledsubsystems} to manipulators with any number of links. But the results assume that the CoM can only be affected by external forces - an assumption invalid when a manipulator (with mass) is in motion. A planar aerial manipulator with any number of rigid or elastic joints was proven to be flat in \cite{flatAM_rigid/elastic}, and such result was generalized to any number of protocentric manipulators in \cite{flatAM_protocentric}; however, the overall CoM of the system must be fixed, else there are unaccounted Coriolis terms.  In \cite{flatquad_suspendedload}, valid flat outputs for rotorcraft with cable-suspended loads were given. But this result does not generalize to aerial manipulators, as  passive cable dynamics cannot model active dynamic coupling in manipulation.  In contrast, our results allow for a completely variable CoM, and arbitrary arm geometries.  

We use recent results in Lagrangian Reduction to formulate reduced aerial manipulator equations of motion (EoM). We have previously used reduction  to prove the small-time locally controllability of an aerial manipulator with a planar arm \cite{wei2021nonlinear}. This work extends the reduction process from a planar arm to general $k$-linked arms. Further, we propose a new flatness proof for aerial manipulators that crucially uses the reduced EoM. The flat outputs are: the linear momentum of the system CoM in the inertial frame, the yaw angle, and the  manipulator joint angles. For completeness, the singularities and the geometric significance of the flatness derivation are addressed. Inspired by the flat system's equivalence to a trivial Brunosky system, we suggest a second order dynamic extension to the thrust input. This allows us to design a locally exponentially stabilizing controller using a control Lyapunov function based quadratic program \cite{SONTAG1989117}. The tracking performance is demonstrated in simulation.

The paper is organized as follows. Section \ref{section:sysdecr} defines an aerial manipulator. Section \ref{sec:Lagrangianreduction} reviews concepts in Lagrangian Reduction, while Section \ref{section:dynamics} develops the reduced equations in a control-affine form.  Section \ref{section:diff_flatness} reviews differential flatness, and proves our main theorem. An exponential stabilizing controller based on flatness and dynamic extension is given in Section \ref{section:controller} and simulated in Section \ref{section:simulation}.

\section{System Description} \label{section:sysdecr}
We analyze the following class of aerial manipulators. 

\noindent \textbf{Definition:}
A multi-rotor aerial platform with the following characteristics is an \textbf{Aerial Manipulator (AM)} (see Fig. \ref{fig:coordinate_frames}):
\begin{itemize}
    \item The multi-rotor includes $n$-pairs of identical rotors attached to a common base, where $n\geq 2$. Each rotor pair consists of one clockwise and one counterclockwise rotating rotor. All thrust axes point in a common direction, denoted by unit vector $\hat{z}_b$. 
    \item A $k$-link fully-actuated manipulator is attached to the base's geometric center. All arm joints are revolute.
    \item All system components are rigid, and complex-fluid structure interactions are ignored.
\end{itemize}
Our model is  derived using the following reference frames:
\begin{itemize}
    \item The earth-fixed inertial frame $E=\{O^e,\hat{x}_e,\hat{y}_e,\hat{z}_e\}$.
    \item The aerial-base body frame $B=\{O^b,\hat{x}_b,\hat{y}_b,\hat{z}_b\}$.
    \item Manipulator $i^{th}$ link frame  $L_i=\{O^{L_i},\hat{x}_{L_i},\hat{y}_{L_i},\hat{z}_{L_i}\}$.
\end{itemize}
Notationally, vector $\mathbold{s}_{ab} \in \mathbb{R}^{3}$ denotes the position of frame B's origin relative to frame A's origin, and and $\mathcal{R}_{ab}\in SO(3)$ denotes the orientation of frame $B$ relative to frame $A$. The geometry of the robot arm is described using the Denavit-Hartenberg (DH) convention. A {\em link reference frame} is attached to each link according to the DH convention. Let $\mathcal{R}_{L_i}$, $\forall i \in \{1,\cdots,k\}$ denote the set of rotation matrices that describe the relative rotation of the $i^{th}$ link frame with respect to the $(i-1)^{st}$ link frame. Let $\mathbold{\eta} \triangleq [\eta_1, \cdots, \eta_k]^{T} \in \mathbb{S}^{k}$ denote the vector of robot arm joint angles, as defined in the DH convention.  The aerial-base body forms link $0$. For simplicity, our derivations assume that $\mathcal{R}_{L_1} = I_{3\times3}$, and that the manipulator's first link operates in the $\hat{x}_b - \hat{z}_b$ plane only. But this assumption is easily generalized. The linear velocity of the aerial-base in the $B$ frame is defined as $\dot{\mathbold{s}}_{b} \triangleq \mathcal{R}_{eb}^T\dot{\mathbold{s}}_{eb}$. 
Using standard roll, pitch, and yaw angles $\mathbold{\xi}=[\phi,\theta,\psi]^T$, the angular velocity, $\mathbold{\omega}_b$, of the aerial-base in the $B$ frame is:
\begin{equation} \label{eq:rollpitchyaw_eq}
    \mathbold{\omega}_{b} = 
   \left[\begin{smallmatrix}
        1 & 0 &-\sin(\theta)  \\
        0 & \cos(\phi) & \sin(\phi)\cos(\theta) \\
        0 & -\sin(\phi) & \cos(\phi)\cos(\theta)
    \end{smallmatrix}\right]
   \left[\begin{smallmatrix}
         \dot{\phi}  \\
         \dot{\theta} \\
         \dot{\psi}
    \end{smallmatrix} \right]
    = \Xi(\mathbold{\xi})\dot{\mathbold{\xi}}.
\end{equation}
Lastly, $\dot{\mathcal{R}}_{eb} \triangleq \mathcal{R}_{eb}S(\mathbold{\omega}_b)$ where $\mathcal{S}(\cdot)$ is the $3\times 3$ skew-symmetric matrix such that  $\mathcal{S}(\mathbold{\omega}_b)\mathbold{\beta} = \mathbold{\omega}_b \times \mathbold{\beta}, \forall \mathbold{\beta} \in \mathbb{R}^{3}$.

\begin{figure}
    \centering
    \includegraphics[width=8.4cm]{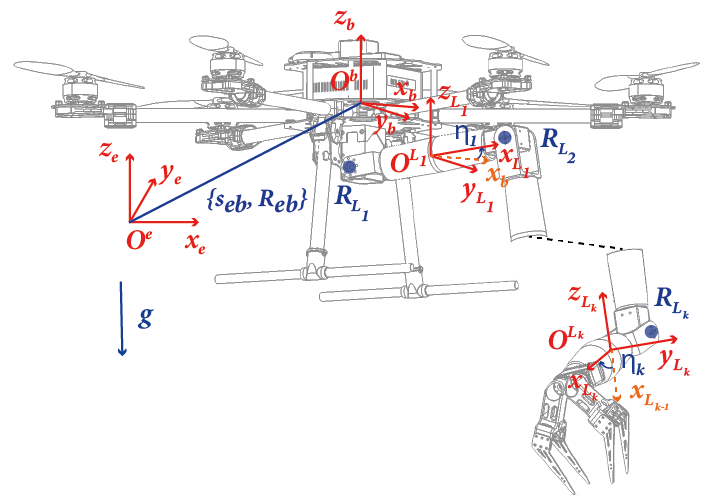}
    \caption{Geometry of an aerial manipulator system.}
    \label{fig:coordinate_frames}
    \vskip -0.2 true in
\end{figure}

\section{Lagrangian Reduction Preliminary} \label{sec:Lagrangianreduction}

A mechanical system is defined by the tuple $\Sigma = (Q,L,\mathcal{T})$, where $Q$ is its finite-dimensional configuration space, assumed to be a smooth manifold.  Let $TQ$ denote the tangent bundle of $Q$, and $T_qQ$ the tangent space to $Q$ at $q\in Q$.  Let $L:TQ\to \mathbb{R}$ be the system Lagrangian and let $\mathcal{T}(q,\dot{q})\in T^*Q$ represent the external forces acting on $\Sigma$, where $T^*Q$ is the dual of $TQ$. A Lagrangian possesses a {\em symmetry} if there is an action on its arguments that renders the Lagrangian invariant. This symmetry allows the reduction of the dynamical system to a lower dimensional phase space. 


A mechanical system possess a symmetry with respect to Lie group $G$ if its Lagrangian $L:TQ\to \mathbb{R}$ and external forces $\mathcal{T}(q,\dot{q})$ are Lie Group Invariant. The {\em left action} of Lie group $G$ on smooth manifold $Q$ is the map $\Phi_{\mathtt{g}}: Q\times Q: \rightarrow Q: q\rightarrow gq$ for any $q \in Q$. The configuration space of an aerial manipulator has the product structure $Q=G\times \mathcal{B}$ where $G=SE(3)$ describes the rigid body location of the multi-rotor base, and the {\em shape space} $\mathcal{B}$ models the manipulator joint variables.  The Lie group $SE(3)$ is a semidirect-product group: $SE(3)= H\otimes \mathcal{V}$, where Lie subgroup $H=SO(3)$ has a left action on  $\mathcal{V}=\mathbb{R}^3$. Thus $Q=SO(3)\otimes \mathbb{R}^3 \times \mathcal{B}$.

Associated with a Lie group is its {\em Lie algebra}, $\mathfrak{g}$, a vector space isomorphic to the tangent space at the group identity, i.e. $\mathfrak{g}\simeq T_{e}G$. 
The Lie algebra of a semidirect-product group can be written as $\mathfrak{g} = \mathfrak{h}\otimes T\mathcal{V} \simeq \mathfrak{h}\otimes \mathcal{V} $ with elements $(\xi_\mathtt{h},\xi_\mathcal{V}) \in \mathfrak{g}$. In local/body coordinates, $\xi_\mathtt{h} = \mathtt{h}^{-1}\dot{\mathtt{h}} \in \mathfrak{h}$ and $\xi_\mathcal{V} = \mathtt{h}^{-1}\dot{v}\in T\mathcal{V}$ where $(\dot{\mathtt{h}},\dot{v})\in TG$ is an arbitrary tangent vector and $\mathfrak{h}$ is the Lie Algebra of group $H$. Hence, $(q,\dot{q}) = (h,\dot{h},\nu,\dot{\nu},r,\dot{r})$ where $h,\dot{h}\! \in\! H$, $\nu,\dot{\nu}\! \in\! \mathcal{V}$, and $r,\dot{r}\!\in\! \mathcal{B}$.  For more details, see \cite{wei2021nonlinear}.

During a manipulation task, the system's CoM displaces as the arm moves, breaking a symmetry in the systems's potential energy.  We use {\em advected parameters} \cite{HOLM19981,Ostrowski_1996} to formulate Lagrangian reduction under symmetry breaking potential energy contributions. For mechanical systems, an {\em advected parameter}, $\mathbold{\gamma}(t)$, is a vector expressed in a body-fixed reference frame satisfying the differential equation:
\begin{equation}\label{eq:advected_def}
    \left(\frac{d}{dt} + \mathtt{g
    }^{-1}(t) \dot{\mathtt{g
    }}(t) \right)\mathbold{\gamma}(t)=0.
\end{equation}
where $ \mathtt{g}\in G$. 
For aerial manipulators, advected parameter $\mathbold{\gamma}(t) \triangleq \mathcal{R}_{eb}^{T}\mathbold{e}_{3}\in \mathcal{V}$ models the direction of gravity (a symmetry-breaking term) in the body-fixed coordinates. The dependency of the potential energy on the aerial base position is another symmetry breaking term. With $\mathbold{\zeta} \triangleq \mathcal{R}_{eb}^T \mathbold{s}_{eb}$, the potential energy $V_{AM}$ can be expressed as $V_{AM}(r,\mathbold{\gamma},\mathbold{\zeta})$, where $r$ denotes the shape variables. Hence, this reformulated potential energy is $G$-invariant.

Define the {\em Augmented Lagrangian} $\overline{L}:TG \times \mathcal{V} \to \mathbb{R}$ by augmenting the state with advected parameter $\mathbold{\gamma}$ and base position $\mathbold{\zeta}$. If the Augmented Lagrangian is $\Phi_{\mathtt{h}}$-invariant, then it can be reduced to $\mathfrak{h}\times \mathcal{V}\times TM$ \cite{Burkhardt_2018} where $\mathfrak{h}$ is a Lie algebra of $H$.  It can be shown that the system's reduced equation takes the general form (see \cite{Burkhardt_2018} Theorem 3.2.1):
\begin{align}\label{eq:matt_eq_1}
        \dot{\varrho}= \dot{r}^{T}\!\alpha(r,\mathbold{\gamma})\dot{r} &+ \dot{r}^{T}\!\beta(r,\mathbold{\gamma})\varrho +\varrho^{T}\!\kappa(r,\mathbold{\gamma})\varrho + \tau_{\varrho}(q,\mathbold{\gamma}),\\
\label{eq:matt_eq_2}
        M(r) \ddot{r} &= -C(r,\dot{r}) + N(r,\dot{r},\varrho) + \tau_{r}(r,\mathbold{\gamma}),\\
\label{eq:matt_eq_3}
        \dot{\mathbold{\gamma}} &= -\xi_h\mathbold{\gamma},\\
        \dot{\mathbold{\zeta}} &= -\xi_h (\mathbold{\zeta} - A(r,\mathbold{\gamma})), \label{eq:matt_eq_4}
\end{align}
where $\varrho\in \mathfrak{se}^{*}(3)$ are generalized momenta, defined as $\varrho_i(\xi) = \left<b_i,\xi \right>$ along the symmetry directions where $\xi\in \mathfrak{g}$ and ${b_i}$ denote any basis of the tangent space to the orbit at $q\in Q$. Also, $\tau_{\varrho}(r,\mathbold{\gamma},\mathbold{\zeta})$ represents the conservative forces and moments resulting from gravity projected along the momenta directions.
The functions $\alpha(r,\mathbold{\gamma})$, $\beta(r,\mathbold{\gamma})$, and $\kappa(r,\mathbold{\gamma})$ are smoothly dependent upon the shape variables $r$ and advected parameter $\mathbold{\gamma}$. 
For the shape dynamics, the reduced mass-matrix, Coriolis, potential terms, and actuation forces are denoted as $M(r)$, $C(r,\dot{r})$, $N(r,\dot{r},\varrho)$, and $\tau_{r}(r,\mathbold{\gamma})$ respectively. Lastly, $A(r,\mathbold{\gamma})$ is an invertible matrix which will arise in the reconstruction equation, defined later (\ref{eq:matt_eq_5}). The structure of (\ref{eq:matt_eq_1}) and  (\ref{eq:matt_eq_2}) follows from the reduced variational principle with
an extended base space consisting of the generalized momenta $\varrho = (\rho_1,\cdots,\rho_n)$ and shape-space variables $(r,\dot{r})$.  Eq.s (\ref{eq:matt_eq_1}) - (\ref{eq:matt_eq_4}) are, respectively, the \textit{momentum equation}, \textit{shape dynamics}, \textit{advection equation}, and \textit{position dynamics} of the system. Together, they form a complete, reduced representation of the system dynamics.

To recover the spatial motion of the system, we employ a {\em reconstruction equation} as known as {\em connection}. It defines a horizontal space of $T_qQ$ as $H_qQ\triangleq\mbox{Ker}(\mathcal{A}(q))$, where $\mathcal{A}$ is a principle connection form and describes motion along the fiber of $Q$ as the flow of a left-invariant vector field. See \cite{BKMM} and \cite{Ostrowski_1996}. The general form of connection is:
\begin{equation} \label{eq:matt_eq_5}
     \xi_{{h}} ={h}^{-1}\dot{{h}} = -A(r,\gamma) \dot{r} + \Gamma^{-1}(r,\gamma)\varrho,
\end{equation}
where $A(r,\gamma)$ and $\Gamma(r,\gamma)$ are mass and inertia liked matrix.

\section{System Dynamics using Lagrangian Reduction and Reconstruction} \label{section:dynamics}

This section explicitly derives the reduced EoM for the general class of aerial manipulators defined above.

\subsection{Kinematics and Dynamics:}\label{sec:kinematicsanddynamics}
Suppose $m_b$ and $\mathcal{I}_b$ are the mass and inertia tensor of the multi-rotor, expressed in frame $B$. The kinetic energy of the multi-rotor is  $K_b = \frac{1}{2}m_b\dot{\mathbold{s}}_{eb}^{T}\dot{\mathbold{s}}_{eb} + \frac{1}{2} \mathbold{\omega}_{b}^T \mathcal{I}_b \mathbold{\omega}_{b}$,
and potential energy is $V_s = m_b g \mathbold{e}_3^T \mathbold{s}_{eb}$, where vectors $\mathbold{e}_1$, $\mathbold{e}_2$, and $\mathbold{e}_3$ are the standard Cartesian basis vectors.

Let the mass of link $i = 1,\cdots, k$, w.r.t. frame $L_k$ be $m_i$.
The $k$-link manipulator dynamics can be conveniently expressed using the manipulator Jacobian matrix 
\cite{MurrayRobotManipulator}. Similar to the multi-rotor base, the kinetic and potential energy of each link can be calculated by summing the translational and rotational contributions.
Let $\dot{\mathbold{x}} = [\dot{\mathbold{s}}_b^T,\mathbold{\omega}_{b}^T,\dot{\mathbold{\eta}}^{T}]^T \in \mathbb{R}^{6+k}$. The total kinetic energy of the system can be rewritten as the following:
\begin{align*} \label{eq:mass_part}
    K_{tot}= \frac{1}{2}  \dot{\mathbold{x}}^{T}\underbrace{\left[\begin{array}{c|c|c}
          \mathcal{M}_{\mathbold{p}} &\mathcal{M}_{\mathbold{p}\mathbold{\omega}} &  \mathcal{M}_{\mathbold{p}\mathbold{l}} \\\hline
          \mathcal{M}_{\mathbold{p}\mathbold{\omega}}^{T} &\mathcal{M}_{\mathbold{\omega}}  & \mathcal{M}_{\mathbold{\omega} \mathbold{l}} \\
          \hline
           \mathcal{M}_{\mathbold{p}\mathbold{l}}^{T} &\mathcal{M}_{\mathbold{\omega} \mathbold{l}}^{T} & \mathcal{M}_{\mathbold{l}}
    \end{array} \right]}_{\text{$\mathcal{M}(\mathbold{\eta})$}}\dot{\mathbold{x}},
\end{align*}
where $\mathcal{M}(\mathbold{\eta})\in \mathbb{R}^{(6+k)\times (6+k)}$  is the overall system mass matrix, and its block partition diagonal matrices $\mathcal{M}_{\mathbold{p}} \in \mathbb{R}^{3\times 3}$, $\mathcal{M}_{\mathbold{\omega}} \in \mathbb{R}^{3\times 3}$,  and $\mathcal{M}_{\mathbold{l}} \in \mathbb{R}^{k\times k}$ are symmetric mass and inertia matrices that describe the multi-rotor structure and the manipulator with respect to $B$ frame. Matrices $\mathcal{M}_{\mathbold{p}\mathbold{\omega}} \in \mathbb{R}^{3\times 3}$, $\mathcal{M}_{\mathbold{p}\mathbold{l}} \in \mathbb{R}^{3\times k}$, and $\mathcal{M}_{\mathbold{\omega} \mathbold{l}} \in \mathbb{R}^{3\times k}$  highlights the coupling effects. 
The total potential energy of the aerial manipulator system w.r.t. frame $E$ is $V_{AM} = g \mathbold{e}_3^T \left(m_t \mathbold{s}_{eb} + \mathcal{R}_{eb}\mathbold{\Delta} \right)$ where $m_t \triangleq m_b +\sum_{i=1}^{k}m_i$ is the total mass of the system, and $\mathbold{\Delta}(\mathbold{\eta}) = \sum_{i=1}^{k}m_i\,s_{bl_i}(\eta_1,\cdots,\eta_i)$ is the manipulator CoM in the frame $B$ . It is important to note that the vector $\mathbold{x}$ contains cyclic coordinates, $\int \mathbold{\dot{s}}_b$, $\int \mathbold{\omega}_b$.

The non-conservative forces (thrust, $T$) and moments (roll, pitch, yaw moment, $\tau_{\phi}$, $\tau_{\theta}$, and $\tau_{\psi}$) produced by the motors are used as control inputs for the multi-rotor system.
The torque input at the revolute joint connecting link $(i-1)$ and link $i$ is denoted as $\tau_{l_{i}}$, $\forall i \in \{1,\cdots,k\}$. As mentioned previously, the manipulator is fully actuated with $k$ total inputs $\mathbold{\tau}_{L} \triangleq[\tau_{l_1},\cdots,\tau_{l_k}]^{T}\in \mathcal{T}_{L}^k \subset \mathbb{R}^{k}$.

\subsection{The Reduced Aerial Manipulator Dynamics}\label{sec:reduceddynamics}
For the aerial manipulator system, the defined advected parameters $\mathbold{\gamma}$ and $\mathbold{\zeta}$ satisfy the following advection equations:
\begin{equation}\label{eq:advectionseq}
    \dot{\mathbold{\gamma}} = -\mathcal{S}(\mathbold{\omega}_b)\mathbold{\gamma} \quad \quad \dot{\mathbold{\zeta}} = -\mathcal{S}(\mathbold{\omega}_b)\mathbold{\zeta} + \dot{\mathbold{s}}_b.
\end{equation}
The conservative forces and momenta due to gravity can be derived using the Lagrange-d'Alembert principle:
\begin{equation} \label{eq:dV}
\begin{array}{ll}
    -dV_{AM} = -\frac{\partial V_{AM}}{\partial \mathbold{\gamma}}\dot{\mathbold{\gamma}}dt -\frac{\partial V_{AM}}{\partial \mathbold{\zeta}}\dot{\mathbold{\zeta}}dt-\frac{\partial V_{AM}}{\partial \mathbold{\eta}}\dot{\mathbold{\eta}}dt.
\end{array}
\end{equation}
Explicitly, the force and torque of gravity $\mathbold{\tau}_{\mathbold{p}}$ and $\mathbold{\tau}_{\mathbold{l}}$ are:
\begin{align}
    \mathbold{\tau}_{\mathbold{p}} = -gm_{t}\mathbold{\gamma}, \quad
    \mathbold{\tau}_{\mathbold{l}}  =-g\mathcal{S}(\mathbold{\gamma})\mathbf{\Delta}
    -g\left(\frac{\partial \mathbold{\Delta}}{\partial \mathbold{\eta}}\right)^{T}\!\mathbold{\gamma}. \label{eq:taul_pt1}
\end{align}
The Augmented Lagrangian, parametrized by $\mathbold{\gamma}$ and $\mathbold{\zeta}$,  is $SE(3)$-invariant by \cite{Burkhardt_2018}, Theorem 3.2.1.
The Augmented Lagrangian  $\overline{L}(\mathbold{\dot{s}}_b, \mathbold{\omega}_b,\mathbold{\eta},\mathbold{\gamma},\mathbold{\zeta})$ for the system is:
\begin{equation} \label{eq:overall_lag}
    \begin{array}{ll}
      \overline{L}\! = \!\frac{1}{2} \left[\begin{smallmatrix}
 \mathbold{\dot{s}}_b\\ \mathbold{\omega}_b\\\mathbold{\eta}
      \end{smallmatrix}\right]^{T}\!\!\mathcal{M}(\mathbf{\mathbold{\eta}})\!\!\left[\begin{smallmatrix}
 \dot{\mathbold{s}}_b\\ \mathbold{\omega}_b\\\mathbold{\eta}
      \end{smallmatrix}\right]
    -g( m_t \langle\mathbold{\gamma},\mathbold{\zeta} \rangle +\langle\mathbold{\gamma}, \mathbold{\Delta} \rangle).
    \end{array}
\end{equation}
In the standard basis for $\mathfrak{se}(3)$, the Lie algebra of $SE(3)$, the generalized linear and angular momenta take the form:  $\mathbold{p}\triangleq \frac{\partial \overline{L}}{\partial \dot{\mathbold{s}}_{b}}\in T^*\mathbb{R}^{3}$ and $\mathbold{l} \triangleq \frac{\partial \overline{L}}{\partial \mathbold{\omega}_b}\in T^*SO(3)$ which are:
\begin{equation} \label{eq:connectioneq_inv}
    \begin{bmatrix}
        \mathbold{p}  \\
        \mathbold{l}
    \end{bmatrix}
    =     
    \underbrace{\begin{bmatrix}
         \mathcal{M}_{\mathbold{p}} & \mathcal{M}_{\mathbold{p}\mathbold{\omega}}    \\
        \mathcal{M}_{\mathbold{p}\mathbold{\omega}}^T & \mathcal{M}_{\mathbold{\omega}}   
    \end{bmatrix}}_\text{$\triangleq \mathcal{M}_s$}   \begin{bmatrix}
         \dot{\mathbold{s}}_b  \\
         \mathbold{\omega}_b
    \end{bmatrix} + \underbrace{
    \begin{bmatrix}
         \mathcal{M}_{\mathbold{p}\mathbold{l}}\\
         \mathcal{M}_{\mathbold{\omega}\mathbold{l}}
    \end{bmatrix}}_\text{$\triangleq \mathcal{M}_{sl}$}     
    \dot{\mathbold{\eta}}.
\end{equation}
The aerial manipulator system's connection (which reconstructs the AM's motion from the momentum equations) is simply:
\begin{equation} \label{eq:connectioneq}
    \begin{bmatrix}
        \dot{\mathbold{s}}_{b} \\
        \mathbold{\omega}_b
    \end{bmatrix}
    =\mathcal{M}_s^{-1}\left(
    \begin{bmatrix}
         \mathbold{p} \\
         \mathbold{l}
    \end{bmatrix} -
   \mathcal{M}_{sl}\dot{\mathbold{\eta}}\right).
\end{equation}
Following the work in \cite{OSTROWSKI1998185}, the non-conservative forces and moments can be projected along the acting momentum directions. Since the symmetry directions are a simple basis of $\mathfrak{se}$(3), the momentum equation including conservative forces $-dV_{AM}$ and non-conservative forces and moments becomes:
\begin{equation}\label{eq:momenta}
\begin{array}{ll}
    \dot{\mathbold{p}} &= \mathbold{p} \times \mathbold{\omega}_b + \mathbold{\tau}_{\mathbold{p}} + T \mathbold{e}_3, \\
    \dot{\mathbold{l}} &= \mathbold{p} \times \dot{\mathbold{s}}_b + \mathbold{l} \times \mathbold{\omega}_b + \mathbold{\tau}_{\mathbold{l}} + 
    \begin{bmatrix}
        \tau_{\phi} &
        \tau_{\theta} &
        \tau_{\psi}
    \end{bmatrix}^{T}. \end{array}
\end{equation}
Further, the shape dynamics of the manipulator is derived based on the Euler-Lagrange equation: 
\begin{equation} \label{eq:ddeta}
   \ddot{\mathbold{\eta}}\! =\! \mathcal{X}\!\left(\!
   - \mathcal{C}(\mathbold{x},\dot{\mathbold{x}})\dot{\mathbold{x}}\! -\! \mathbold{D}(\mathbold{x})\! +\! 
   \left[ \begin{smallmatrix}
   \mathbf{0}_{1 \times 2} & T & \tau_{\phi} & \tau_\theta & \tau_\psi & \mathbold{\tau}_{L}^{T}
   \end{smallmatrix}\right]^T\right),
\end{equation}
where matrix $\mathcal{X} \triangleq \left[\begin{smallmatrix}
   0_{k\times 6} & I_{k\times k}
   \end{smallmatrix} \right]\mathcal{M}^{-1}(\mathbold{\eta})$, and $\mathcal{C}(\mathbold{x},\dot{\mathbold{x}})$ is the Coriolis matrix of the shape variable which can be calculated as the following \cite{MurrayRobotManipulator}:
\begin{equation} \label{eq:coriolismatrix}
    \mathcal{C}_{(p,j)} \!=\! \sum_{i=1}^{12+2k}\!\frac{1}{2}\left(\frac{\partial\mathcal{M}_{(p,j)}}{\partial x_{i}} + \frac{\partial\mathcal{M}_{(p,i)}}{\partial x_{j}} - \frac{\partial\mathcal{M}_{(i,j)}}{\partial x_{p}}  \right),
\end{equation}
and $\mathbold{D}(\mathbold{x})$ is the potential term where $\mathbold{D}(\mathbold{x}) = \frac{\partial V_{AM}}{\partial \mathbold{x}}$. 

\subsection{Control Affine Form of the AM's reduced dynamics} \label{sec:controlaffineform}
We choose the following state parametrization for the reduced dynamical equations: 
\begin{equation*} \label{eq:systemstatevariables}
    \mathbf{q} \triangleq [\mathbold{p}^{T},\mathbold{l}^{T},\phi,\theta,\psi,\mathbold{\eta}^{T},\dot{\mathbold{\eta}}^{T}]^{T}\in \mathbb{R}^{9+2k},
\end{equation*}
where Euler angles are used, instead of the advected parameter $\mathbold{\gamma}$, to parametrize  multi-rotor orientation.  With $\mathbf{q}$, the reduced AM dynamics take the control affine form:
\begin{align}
 \label{eq:control_affine_form_final}
    \dot{\mathbf{q}} &=  \mathbf{f}(\mathbf{q}) + \mathcal{G}(\mathbf{q})\mathbf{u},\\
    \mathbf{u} &= [\tau_{l_1},\cdots,\tau_{l_k},T,\tau_{\phi},\tau_{\theta},\tau_{\psi}]^{T}\in \mathcal{U}\subset \mathbb{R}^{4+k},
\end{align} 
where $\mathcal{U}$ is the space of all control inputs.
Further, the drift term and input vector fields are
\begin{align*} \label{eq:affinef}
    \mathbf{f}(\mathbf{q}) &= 
    \begin{bmatrix}
        \mathbold{p} \times \mathbold{\omega}_b - m_t g \mathbold{\gamma} \\
        \mathbold{p} \times \dot{\mathbold{s}}_{b} + \mathbold{l} \times \mathbold{\omega}_b -g\mathcal{S}(\mathbold{\gamma})\mathbf{\Delta}
    -g\left(\frac{\partial \mathbold{\Delta}}{\partial \mathbold{\eta}}\right)^{T}\mathbold{\gamma}\\
         \Xi^{-1}(\mathbold{\xi}) \mathbold{\omega}_{b}(\mathbold{p},\mathbold{l},\mathbold{\eta},\mathbold{\dot{\eta}})\\
        \dot{\mathbold{\eta}}\\
         \mathcal{X} \left(- \mathcal{C}(\mathbold{x},\dot{\mathbold{x}})\dot{\mathbold{x}} - \mathbold{D}(\mathbold{x})\right)
    \end{bmatrix},\\ 
\mathcal{G}(\mathbf{q}) &= \begin{bmatrix}
    0_{3\times k} & \mathbold{e}_{3} & 0_{3\times 3} \\
    0_{3\times k} & 0_{3\times 1} & I_{3\times 3}\\
    0_{3\times k} & 0_{3\times 1}& 0_{3\times 3}\\
    0_{2\times k} & 0_{2\times 1}& 0_{2\times 3}    \\
   \mathcal{X}\left[\begin{smallmatrix}0_{6\times k} \\I_{k\times k} \end{smallmatrix}\right] & \mathcal{X}\left[\begin{smallmatrix}0_{2\times 1}\\1\\0_{(3+2k)\times 1} \end{smallmatrix}\right] &\mathcal{X}\left[\begin{smallmatrix}0_{3\times 3}\\I_{3\times 3}\\0_{k\times 3} \end{smallmatrix}\right]\
\end{bmatrix}.
\end{align*}
It is important to note, states $\mathbold{\omega}_{b}$ and $\dot{\mathbold{s}}_b$ can be expressed in terms of $\mathbf{q}$ using the connection (\ref{eq:connectioneq}).

\section{Differential Flatness} \label{section:diff_flatness}

Flatness, first defined by Fliess et al. \cite{Fliess} and originating from differential algebra, transforms systems as a differential field generated by a set of states and inputs. A flat system has well characterized nonlinear structures which can be exploited in designing control algorithms for planning, trajectory generation, and stabilization \cite{equivalentsystems}. We here adopted the Lie-Bäcklund framework to approach to flatness and system equivalence. Consider two systems ($\mathcal{A}$, $X$) and ($\mathcal{B}$, $Y$) and a smooth mapping $\Phi:\mathcal{A}\to \mathcal{B}$. The pair $(\mathcal{A},X)$ is a system of differential equation where $\mathcal{A}$ is an open set of $\mathbb{R}^{n}$ and $X$ is a smooth vector field on $\mathcal{A}$.
\begin{definition}(Equivalent System)  \cite{equivalentsystems}
    Systems ($\mathcal{A}$, $X$) and ($\mathcal{B}$, $Y$) are equivalent at $(a,b)\in \mathcal{A}\times \mathcal{B}$ if there exist a smooth mapping $\Phi$ from a neighborhood of $a$ to a neighborhood of $b = \Phi(a)$ which is an endogenous transformation at $(a,b)$.
\end{definition}
An endogenous transformation is an invertible transformation that "exchanges" the trajectory between two systems. See \cite{equivalentsystems} for the rigorous definition. This leads to the formal definition of differential flatness:
\begin{definition} (Differentially Flat System)$\,$
    The control system $(\mathcal{A},X)$ is differentially \textit{flat} around $a$ if and only if it is equivalent to a \textit{trivial system} in a neighborhood of $a$.
\end{definition}
The \textit{trivial system} referred to in the above definition is the system $(\mathbb{R}^{\infty}_{a},X_{a})$ with coordinates $(\mathbold{y},\dot{\mathbold{y}},\ddot{\mathbold{y}},\ldots)$ and vector field $X_{a}(\mathbold{y},\dot{\mathbold{y}},\ddot{\mathbold{y}},\ldots) = (\dot{\mathbold{y}},\ddot{\mathbold{y}},\dddot{\mathbold{y}},\ldots)$. Casually speaking, the trivial system composes of chain of integrators. 
Further, the set $\mathbold{y} = \{y_{j} \text{ s.t. } j = 1,\cdots, a\}$ is called a \textit{flat output} of $\mathcal{A}$. This is equivalent to the more familiarizing yet informal definition.
Given the nonlinear system
\begin{equation}
    \label{eq:flatdefinition_def}
    \dot{\mathbold{x}} = f(\mathbold{x}(t),\mathbold{u}(t));\; \mathbold{x} \in \mathbb{R}^n, \mathbold{u} \in \mathbb{R}^m,
\end{equation}
where $\mathbold{x}$ are the $n$ states and $\mathbold{u}$ are the $m$ inputs, $\mathbold{y} \in \mathbb{R}^m$ is said to be a \textit{flat output} if:
\begin{itemize}
    \item $y_i = g_i(\mathbold{x},\mathbold{u},\dot{\mathbold{u}},\ldots,\mathbold{u}^{(j_i)}), \; j_i \in \mathbb{N}, i = 1, 2,\ldots , m$.
    \item $x_i = h_i(\mathbold{y}, \dot{\mathbold{y}},\ldots, \mathbold{y}^{(k_i)}),\; k_i \in \mathbb{N}, i = 1, 2,\ldots, n$, \\ 
          $ u_i = \Tilde{h}_i(\mathbold{y}, \dot{\mathbold{y}}, \ldots, \mathbold{y}^{(l_i)}), \; l_i \in \mathbb{N}, i = 1, 2,\ldots , m$.
    \item All components of $\mathbold{y}$ are differentially independent, i.e. $\mathbold{y}$ satisfies no differential equation $\Phi(\mathbold{y}, \dot{\mathbold{y}}, \ldots, \mathbold{y}^{(k)}) = 0, \; k \in \mathbb{N}$.
\end{itemize}

\subsection{Main Theorem and Proof}
\begin{theorem}
$\mathbold{\sigma} = [(\mathcal{R}_{eb}\mathbold{p})^{T},\psi,\mathbold{\eta}^{T}]^{T} = [\sigma_1, \cdots,\sigma_{4+k}]^T \in \mathbb{R}^{4+k}$ is set of differentially flat outputs for the defined class of aerial manipulators except at singularities which are $\phi, \theta = \frac{\kappa \pi}{2}$, $\forall \kappa \in \mathbb{Z}$ and $T = 0$.  
\end{theorem}
\begin{proof}
For ease of notation, we define $\mathcal{R}_{eb}\mathbold{p} \triangleq  \mathbold{p}_{e}$.
Physically speaking, the flat outputs consists of $\mathbold{p}_{e}$, the linear momentum of the CoM in $E$ frame, $\psi$, yaw position angle, and $\mathbold{\eta}$, relative joint angles. We will state without proving that these outputs are differentially independent except at singularity points addressed at the end of this section. 

Starting with extracting the only external force, thrust $T$, we can differentiate $\mathbold{p}_{e}$ that unfolds the following relationship:
\begin{equation} \label{eq:Rebp}
    \frac{d}{dt}\mathbold{p}_{e}= \mathcal{R}_{eb} \left(\dot{ \mathbold{p}} - \mathbold{p} \times \mathbold{\omega}_{b} \right) = -m_t g \mathbold{e}_{3}+\mathcal{R}_{eb} \mathbold{e}_{3} T.
\end{equation}
Algebraically, we can use \eqref{eq:Rebp} to extract the forces and orientation (represented using Euler angles):
\begin{equation} \label{eq:T_insigma}
    T(\dot{\mathbold{p}}_{e})= \|\dot{\mathbold{p}}_{e} + m_tg\mathbold{e}_3 \|_2 ,
\end{equation}
\begin{equation} \label{eq:phi_insigma}
    \phi(\dot{\mathbold{p}}_{e},\psi) = \sin^{-1}\left(\frac{\mathbold{e}_1^{T}\dot{\mathbold{p}}_{e}\sin(\psi) - \mathbold{e}_2^{T}\dot{\mathbold{p}}_{e} \cos\psi}{T} \right),
\end{equation}
\begin{equation}\label{eq:theta_insigma}
    \theta(\dot{\mathbold{p}}_{e},\psi) = \tan^{-1}\left(\frac{\mathbold{e}_1^{T}\dot{\mathbold{p}}_{e}\cos(\psi) + \mathbold{e}_2^{T}\dot{\mathbold{p}}_{e} \sin\psi}{\mathbold{e}_3^{T}\dot{\mathbold{p}}_{e} + m_tg} \right).
\end{equation}
Therefore, $\mathcal{R}_{eb}(\mathbold{\xi}) = \mathcal{R}_{eb}(\dot{\mathbold{p}}_{e},\psi)$.

Using the roll-pitch-yaw dynamics \eqref{eq:rollpitchyaw_eq}, the body rates can be computed as $\mathbold{\omega}_b = \Xi(\mathbold{\xi})\dot{\mathbold{\xi}} = \mathbold{\omega}_b (\dot{\mathbold{p}}_{e},\ddot{\mathbold{p}}_{e},\psi,\dot{\psi})$.
We can also recover the body frame general linear momenta $\mathbold{p}$ as a function as the flat outputs and their derivatives once $\mathcal{R}_{eb}$ is known, $\mathbold{p}(\mathbold{p}_{e},\dot{\mathbold{p}}_{e},\psi) = \mathcal{R}_{eb}^{T}\mathbold{p}_{e}$.
Using the connection \eqref{eq:connectioneq}, we can obtain the multi-rotor linear velocity in frame $B$:
\begin{equation*} \label{eq:sbdot_insigma}
    \dot{\mathbold{s}}_{b}\left(\mathbold{p}_e, \dot{\mathbold{p}}_{e}, \ddot{\mathbold{p}}_{e}, \psi, \dot{\psi}, \mathbold{\eta}, \dot{\mathbold{\eta}}\right) = \frac{1}{m_t}\left(\mathbold{p} - \mathcal{M} _{\mathbold{{p\omega}}}\mathbold{\omega}_{b}-\mathcal{M}_{\mathbold{pl}}\dot{\mathbold{\eta}}\right).
\end{equation*}
From the definition of generalized linear momentum \eqref{eq:connectioneq_inv}, the overall CoM angular momenta, $\mathbold{l}$, and its derivative, $\dot{\mathbold{l}}$, are also functions of the flat outputs and their time derivatives:
\begin{equation*} \label{eq:l_insigma}
    \mathbold{l}(\mathbold{p}_e, \dot{\mathbold{p}}_{e}, \ddot{\mathbold{p}}_{e},\psi,\dot{\psi},\mathbold{\eta},\dot{\mathbold{\eta}})  = \mathcal{M}_{\mathbold{p}\mathbold{\omega}}^{T}\dot{\mathbold{s}}_b + \mathcal{M}_{\mathbold{\omega}}\mathbold{\omega}_b + \mathcal{M}_{\mathbold{\omega} \mathbold{l}}
    \dot{\mathbold{\eta}}.
\end{equation*}
Compactly, $\dot{\mathbold{l}} =\dot{\mathbold{l}}(\mathbold{p}_e, \dot{\mathbold{p}}_{e}, \ddot{\mathbold{p}}_{e}, \dddot{\mathbold{p}}_{e},\psi,\dot{\psi},\ddot{\psi},\mathbold{\eta},\dot{\mathbold{\eta}},\ddot{\mathbold{\eta}})$. Lastly, we have 3 equations and 3 unknowns to algebraically solve for the roll-pitch-yaw torque $\tau_\phi, \tau_\theta$ and $\tau_\psi$:
\begin{equation} \label{eq:torquesrollpitchyaw_insigma}
    \begin{bmatrix}
         \tau_{\phi}\\
         \tau_{\theta}\\
         \tau_{\psi}
    \end{bmatrix} 
    \left(\begin{smallmatrix}\mathbold{p}_e, \dot{\mathbold{p}}_{e}, \ddot{\mathbold{p}}_{e},\\ \dddot{\mathbold{p}}_{e},\psi,\dot{\psi},\\\ddot{\psi},\mathbold{\eta},\dot{\mathbold{\eta}},\ddot{\mathbold{\eta}}\end{smallmatrix}\right) 
    = \dot{\mathbold{l}} -\mathbold{l}\times \omega_b - \mathbold{\tau}_{\mathbold{l}} - \mathbold{p} \times \dot{\mathbold{s}}_{b}.
\end{equation}
The torque input for the manipulator can be calculated by substituting the respective variables into \eqref{eq:ddeta}:
\begin{equation} \label{eq:manipulatortorques_insigma}
    \mathbold{\tau}_{\mathbold{L}} = \left[\begin{smallmatrix}
   0_{k\times 6} & I_{k\times k}
   \end{smallmatrix} \right]\left(\mathcal{M}(\mathbold{\eta})\ddot{\mathbold{x}} + \mathcal{C}(\mathbold{x},\dot{\mathbold{x}})\dot{\mathbold{x}} + \mathbold{D}(\mathbold{x})\right),
\end{equation}
where $\ddot{\mathbold{x}} = \begin{bmatrix}
    \ddot{\mathbold{s}}_b^T &
    \dot{\mathbold{\omega}}_b^T &
    \ddot{\mathbold{\eta}}
    \end{bmatrix}^T$.
    By differentiation, $\ddot{\mathbold{s}}_b$ and $\ddot{\mathbold{\omega}}_b$ are also functions of the flat outputs and their derivatives up to $\dddot{\mathbold{p}}_{e}$, $\ddot{\psi}$, and $ \ddot{\mathbold{\eta}}$.
As a result, one can show that $\mathbold{\tau}_{L} = \mathbold{\tau}_{L}(\mathbold{p}_e, \dot{\mathbold{p}}_{e}, \ddot{\mathbold{p}}_{e}, \dddot{\mathbold{p}}_{e},\psi,\dot{\psi},\ddot{\psi},\mathbold{\eta},\dot{\mathbold{\eta}},\ddot{\mathbold{\eta}})$. Therefore, all states $\mathbold{q}$ and inputs $\mathbold{u}$ are algebraic functions of the flat outputs $\mathbold{\sigma}$ and their time derivatives.
\end{proof}

The only singularities arising in the proof are when the system at free fall, i.e. $\mathbold{e}_3^{T}\dot{\mathbold{p}}_{e} + m_tg= 0$, or equivalently $T\cos\phi \cos\theta  = 0$.
Therefore, we confine the domain of the attitude angles to avoid singularity, i.e. $\phi \in (-\frac{\pi}{2},\frac{\pi}{2})$, $\theta \in (-\frac{\pi}{2},\frac{\pi}{2})$, and thrust to be strictly positive, $T>0$.


\subsection{Equivalence to trivial system:}
Similar to the multi-rotor case, 
with a second order dynamic extension of the total thrust, we can obtain a valid relative degree using the flat output as the desired output \cite{flatnessMPC}. We extend the state variable to include the thrust dynamics:
$\mathbf{q}_{de} \triangleq \begin{bmatrix}
    \mathbf{q}^{T},T,\dot{T}
\end{bmatrix}^{T}$. Consequently, the control inputs become $\mathbf{u}_{de} = \begin{bmatrix}
     \mathbold{\tau}_{L}^{T},\ddot{T},\tau_{\phi},\tau_{\theta},\tau_{\psi}
\end{bmatrix}^{T}\in \mathcal{U}_{de}$.
We here restate the control-affine form with extended states and inputs. 
\begin{align} \label{eq:CAF_de}
    \dot{\mathbf{q}}_{de} = \mathbf{f}_{de}(\mathbf{q}_{de}) &+ \mathcal{G}_{de}(\mathbf{q}_{de})\mathbf{u}_{de},\\
    \mathbf{f}_{de} \!= \!\!
    \left[\begin{smallmatrix}
         \!\mathbf{f}(\mathbf{q})+\left[\begin{smallmatrix}
         \mathbold{e}_{3}T\\
         0_{8\times1}\\
         \mathcal{X}\left[\begin{smallmatrix}0_{2\times 1}\\T\\0_{(3+2k)\times 1} \end{smallmatrix}\right]
    \end{smallmatrix}\right]\\
         \dot{T}\\
         0
    \end{smallmatrix} \!\right]&\!,\,
    \mathcal{G}_{de} \!=\!\!\left[\begin{smallmatrix}
    0_{1\times (k+4)}
      \\ \mathbf{\mathcal{G}(\mathbf{q})}_{(2:11,:)} \\
        \!\mathcal{X}\left[\begin{smallmatrix}0_{3\times k} & 0_{3\times 1} & 0_{3\times 3}\\
        0_{3\times k} & 0_{3\times 1} & I_{3\times 3}\\
        I_{k\times k} & 0_{k\times 1} & 0_{k\times 3}\\
        \end{smallmatrix}\right] \\
         0_{1\times (k+4)}  \\
         [0_{1\times k}  \quad 1\quad   0_{1\times3}]
    \end{smallmatrix} \!\right]\!\!. \nonumber
\end{align}
Since the original inputs $\mathbf{u}$ are functions of the flat output up to their $3^{rd}$ time derivatives, we propose an auxiliary control input $\mathbf{v} = \left[\begin{smallmatrix}
  \left(\mathbold{p}^{T}_{e}\right)^{(3)} &
  \psi^{(2)}&
  \left(\mathbold{\eta}^{T}\right)^{(2)}
  \end{smallmatrix}\right]^{T}$, where the $^{(\cdot)}$ denotes the number of time differentiation, which can expressed as the following:
\begin{equation} \label{eq:prefeedbacklinearizedfrom}
    \mathbf{v}= 
   \mathbf{f}_v(\mathbold{\sigma},\dot{\mathbold{\sigma}},\ddot{\mathbold{\sigma}},T,\dot{T}) + \mathcal{G}_v(\mathbold{\sigma},\dot{\mathbold{\sigma}},\ddot{\mathbold{\sigma}},T,\dot{T}) \mathbf{u_{de}},
\end{equation}
where $\mathbf{f}_v\in \mathbb{R}^{4+k}$ and $\mathcal{G}_{v} \in \mathbb{R}^{(4+k) \times (4+k)}$ are the drift vector field and actuation matrix, respectively. Using the extended dynamics \eqref{eq:CAF_de} and connection \eqref{eq:connectioneq_inv}, one can show $\mathbf{f}_v$ and $\mathcal{G}_{v}$ can be expressed as a function of $\mathbf{q}_{de}$ by differentiating the flat outputs until the control inputs surfaces.

By the system equivalence argument, one can obtain the following Brunovsky's canonical description of \eqref{eq:control_affine_form_final}:
\begin{multline}
    \label{eq:brunovskyscanonicalform}
    \frac{d}{dt}\!\!\!\left[\!\begin{smallmatrix}
         \mathbold{p}_{e}\\
         \psi \\
         \mathbold{\eta} \\
         \dot{\mathbold{p}}_{e}
         \\
         \dot{\psi}\\
     \dot{\mathbold{\eta}}\\
        \ddot{\mathbold{p}}_{e}
    \end{smallmatrix} \!\right]\!\!  =\!\underbrace{
    \left[\!\begin{smallmatrix}
        0_{(7\!+\!k)\times (4\!+\!k)} & I_{(7\!+\!k)\!\times\! (7\!+\!k)} \\
        0_{(4\!+\!k)\!\times \!(4\!+\!k)} & 0_{(4\!+\!k)\!\times\! (7\!+\!k)}
    \end{smallmatrix}\!\right]}_{\text{$\mathcal{F}_{B}$}}\!\!
    \left[\!\begin{smallmatrix}
         \mathbold{p}_{e}\\
         \psi \\
         \mathbold{\eta} \\
         \dot{\mathbold{p}}_{e}
         \\
         \dot{\psi}\\
     \dot{\mathbold{\eta}}\\
        \ddot{\mathbold{p}}_{e}
    \end{smallmatrix} \!\right]\! +\!\underbrace{\left[\!\begin{smallmatrix}
         0_{(7\!+\!k)\!\times\! (4\!+\!k)}  \\
         I_{(4\!+\!k)\!\times\! (4\!+\!k)}
    \end{smallmatrix}\! \right]}_{\text{$\mathcal{G}_ B$}}\!\mathbf{v}.
\end{multline}

\begin{figure*}[ht!]
    \centering
    \subfloat[Desired Output Tracking Performance and Min-Norm control inputs using CLF-QP controller from Simulation. $\mathcal{Q}$ is a $15$ by $15$ identity matrix to maximize radius of attraction. An initial condition disturbance of is introduced to demonstrate robustness. Linear momentum $\sigma_1$, $\sigma_2$, and $\sigma_3$ are in $kg\cdot m/s$ and angles $\psi$, $\eta_1$, and $\eta_2$ are in $rad$. ((b)) CLF-QP and theoretical control input are compared where the theoretical ones are obtained from flatness derivation: \eqref{eq:torquesrollpitchyaw_insigma}, \eqref{eq:manipulatortorques_insigma}, extend thrust using given desired trajectory. Torque inputs $\tau_{l_1}$, $\tau_{l_2}$, $\tau_{\phi}$, $\tau_{\theta}$, and $\tau_{\psi}$ are in $N\cdot m$, and $\ddot{T}$ has the unit of $N/s^2$.  ((c)) Visualization of aerial manipulator follows a desired path and stabilizes while the arm is manipulating.]{%
      \includegraphics[width=17cm]{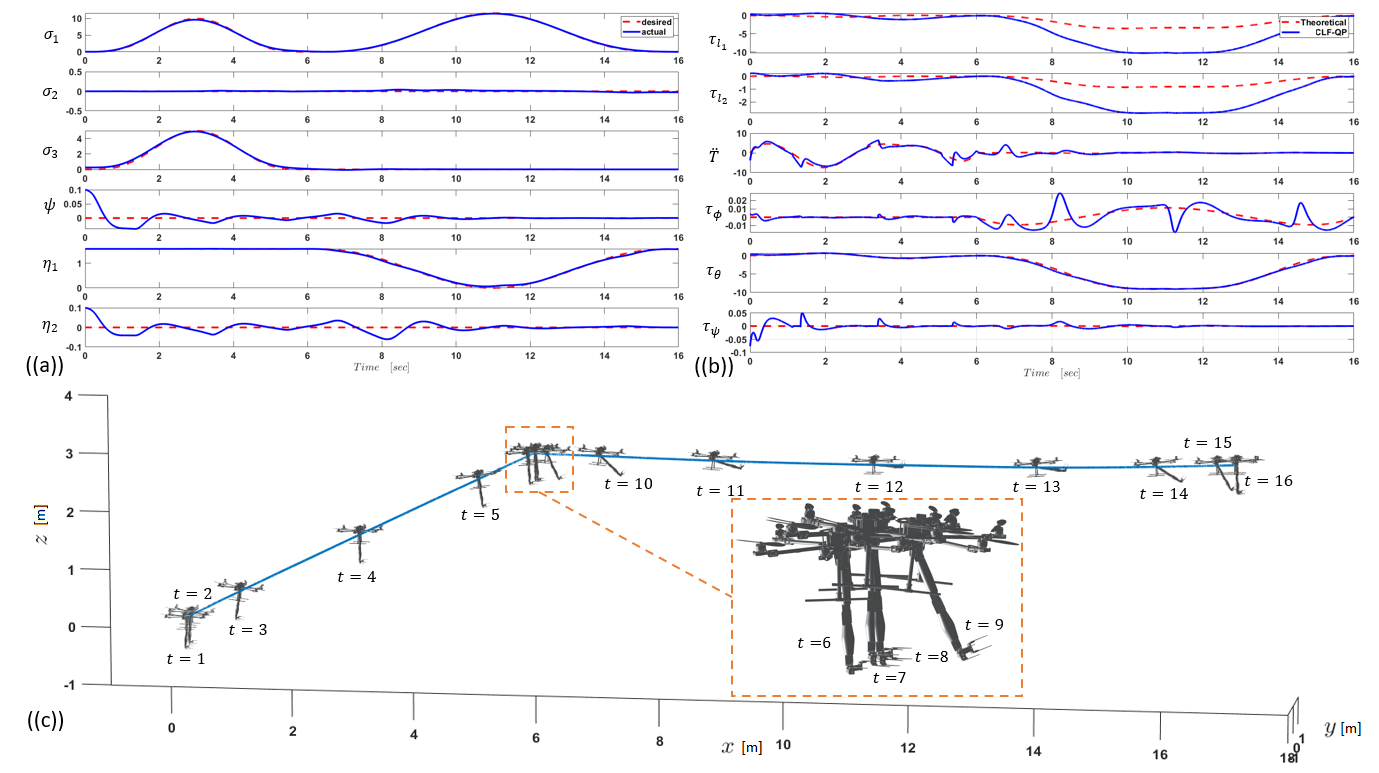}
    }
\vspace{0.1cm}
\caption{Simulation result of an aerial manipulator with a 2-link manipulator.}
\label{fig:simresults}
\end{figure*}

\subsection{Results Discussion}


Our general class of aerial manipulators are strongly and nonlinearly {\em accessible} and {\em controllable} because of flatness \cite{Fliess}.  The flatness proof is rather algebraic, which leads to efficient implementations, but masks its geometric significance.  One can observe similarities between the AMs analyzes in this paper and a the well known proof of  multi-rotor flatness \cite{flatquad1}. 
Our method of determining the rotation matrix $\mathcal{R}_{eb}$ is identical to the multi-rotor case, but uses an Euler representation of $SO(3)$ representation. Inverting the yaw rotation by $\psi$, the directional vector parallel to gravity in body frame gives us the roll and pitch angle as shown in \eqref{eq:phi_insigma} and \eqref{eq:theta_insigma}. A key difference arises in our use of the reconstruction equation \eqref{eq:connectioneq}, which models the coupling between the arm dynamics and multi-rotor base dynamics in symbolically compact way.

The dynamic extension is non-conventional, but practically motivated. It can be realize by considering a standard DC motor model. Let $\Omega_j$ be the RPM of the $j^{th}$ motor. The thrust and resistive torque generated by the $j^{th}$ motor can be modeled as $c_{T}\Omega_{j}^2$ and $\pm c_{Q}\Omega_{j}^2$, respectively, where $c_{T}$ and $c_{Q}$ are coefficients that depends on rotor geometry. Further, the rotor acceleration per minute, $\dot{\Omega}_j$, is proportional to the motor axial torque as well as the armature current as $i(t) = \frac{\dot{\Omega}_j}{K_{m_j}}$. Using a standard RLC circuit model, we can express $\ddot{T}_j$ as:
\begin{align*}
    \ddot{T}_j = 2c_{T}\frac{i^2(t)}{K^2_{m_j}} + \frac{2c_{T}}{K_{m_j}L_j}\Omega_j\left(U_j - K_{\mathcal{V}_j}\Omega_{j}-R_j\,i(t) \right),
\end{align*}
where $U_j$, $R_j$, $L_j$, $K_{m_j}$ and $K_{\mathcal{V}_j}$ are the voltage, internal resistance, inductance, torque constant, and back-emf constant of brushless motor $j$, respectively. Thus, $\ddot{T}$ can be modulated by regulating the electrical power feed to the motors.

\section{Exponentially Tracking Controllers} \label{section:controller}

Like multi-rotors, AMs are underactuated. Further, their CoMs can shift signficantly during arm manipulation. These characteristics offer challenges to the stabilization and trajectory tracking problems of  AMs. Suppose we are given a desired trajectory $\mathbold{\sigma}_{d}(t) = [\mathbold{p}_{e,d}^{T}(t),\psi_{d}(t),\mathbold{\eta}_{d}^{T}(t)]^{T}$, as a function of time. We assume the path be at least three times continuously differentiable and dynamically feasible (i.e., desired behavior can be realized within the range of control inputs). We can exploit the Brunovsky's equivalent system, a readily fully controllable normal form, as in \cite{flatnessMPC} and \cite{non_min_Phase_example}. Hence, we use this dynamic feedback linearized form \eqref{eq:control_affine_form_final} to implement an optimization based controller that guarantees local exponential tracking. 
 
Starting with defining  the tracking error between the actual output and the time-dependent desired trajectory:
\begin{equation}
    \mathbf{e}(\mathbf{q}_{de})\triangleq \begin{bmatrix}
    \mathbold{p}_{e}(\mathbf{q}_{ed})\\
    \hline
    \psi(\mathbf{q}_{ed})\\
    \mathbold{\eta}(\mathbf{q}_{ed})
    \end{bmatrix}-  \begin{bmatrix}
    \mathbold{p}_{e,d}(t)\\ \hline
    \psi_{d}(t)\\
    \mathbold{\eta}_{d}(t)
    \end{bmatrix} = \begin{bmatrix}
         \mathbf{e}_1(\mathbf{q}_{de}(t))\\ \hline \\
         \mathbf{e}_2(\mathbf{q}_{de}(t))\\
    \end{bmatrix}.
\end{equation}
 One can easily verify that the output tracking error $\mathbf{e}(\mathbf{q}_{de})$ has a vector relative degree $\mathbf{r} =\left[\begin{smallmatrix}
      3 & 3 & 3 & 2 & 2 & \cdots& 2
 \end{smallmatrix}\right]$. Explicitly,
 \begin{equation}
     \mathbf{e}^{(\mathbf{r})}\!(\mathbf{q}_{de}) = \begin{bmatrix}
          \mathbf{e}_{1}^{(3)}(\mathbf{q}_{de})\\
          \mathbf{e}_{2}^{(2)}(\mathbf{q}_{de})
     \end{bmatrix}= \mathbf{f}_{\mathbf{v}}(\mathbf{q}_{de}) - \mathbold{\sigma}_{d}^{\mathbf{(r)}} + \mathcal{G}_{\mathbf{v}}\mathbf{u}_{de}.
 \end{equation}
The rank of the decoupling matrix  $\mathcal{G}_{\mathbf{v}}$ is verified symbolically, indicating that matrix $\mathcal{G}_{\mathbf{v}}$ is invertible except at the defined singularities. We leverage the well-established CLF-QP formulation \cite{SONTAG1989117},\cite{ames2014rapidly},\cite{GRIZZLE20141955} to drive the error, $\mathbf{e}(\mathbf{q_{de})}$, to zero exponentially. Let
$\mathbf{h}(\mathbf{q_{de},t}) = \begin{bmatrix}
    \mathbf{e}_1^{T} & \mathbf{e}_2^{T} &
     \dot{\mathbf{e}}_1^{T} & \dot{\mathbf{e}}_2^{T} & \ddot{\mathbf{e}}_1^{T} 
\end{bmatrix}^{T}$, $\dot{\mathbf{h}} = \mathcal{F}_{B}\mathbf{h} + \mathcal{G}_{B}\mathbf{v}$ where $\mathcal{F}_B$,  $\mathcal{G}_{B}$, and $\mathbf{v}$ are stated in \eqref{eq:brunovskyscanonicalform}. We construct the CLF, $V(\mathbf{h})$, using the solution $\mathcal{P} = \mathcal{P}^{T}\succ 0$ of the Continuous-Time Algebraic Riccati Equation (CARE):
\begin{align} \label{eq:care}
    \!V(\mathbf{h}) = \mathbf{h}^{T}\mathcal{P}\mathbf{h}, \;\;
    \mathcal{F}_{B}^{T}\mathcal{P} + \mathcal{P}^{T}\mathcal{F}_{B} - \mathcal{P}\mathcal{G}_{B} \mathcal{G}_{B}^{T}\mathcal{P}  = - \mathcal{Q},
\end{align}
where $\mathcal{Q} = \mathcal{Q}^{T} \succ 0 $. Inspired by \cite{ames2014rapidly}, we propose the following task space QP controller:
\small
\begin{align} \label{eq:ES-CLF-QP}
\begin{split}
    \mathbf{u}_{de}^{*} &= \argmin_{\mathbf{u}_{de} \in \mathcal{U}_{de} \subset \mathbb{R}^{4+k}} \|\mathcal{G}_{v}(\mathbf{q}_{de}) \mathbf{u}_{de} +\mathbf{f}_{v}(\mathbf{q}_{de})\|^2,\\
    \mbox{s.t.}\! \quad L_{\mathcal{F}_B}\!&V(\mathbf{q}_{de},t) + L_{\mathcal{G}_B}\!V(\mathbf{q}_{de},t)(\mathcal{G}_{de}(\mathbf{q}_{de})\mathbf{u}_{de}+\mathbf{f}_{\mathbf{v}}(\mathbf{q}_{de})) \\ 
    &-L_{\mathcal{G}_B}\!V(\mathbf{q}_{de},t)\mathbold{\sigma}_{d}^{\mathbf{(r)}}
    \leq -\frac{\lambda_{min}(\mathcal{Q})}{\lambda_{max}(\mathcal{P})}V(\mathbf{q}_{de},t),
\end{split}
\end{align}
\normalsize
where $L_{\mathcal{F}_B}\!V$ and $L_{\mathcal{G}_B}\!V$ are the following
\begin{align*}
    \dot{V}(\mathbf{h}) =\underbrace{\mathbf{h}^T(\mathcal{F}_{B}^{T}\mathcal{P} + \mathcal{P}\mathcal{F}_{B})\mathbf{h}}_{L_{\mathcal{F}_{B}}V}+\underbrace{2\mathbf{h}^{T}\mathcal{P}\mathcal{G}_{B}}_{L_{\mathcal{G}_{B}}V}\mathbf{v}<0, \quad   \forall \mathbf{h},
\end{align*}
One key advantage of  the exponentially stabilizing CLF-QP controller \eqref{eq:ES-CLF-QP} is its ability to incorporate nonholonomic constraints. E.g., it can include a friction cone constraint to model the contact between the arm's end-effector and a surface, or the act of AM perching. On the other hand, if input constraints are not incorporated in the trajectory planning process, they can be added into the QP formulation at the cost of a lower exponential convergence rate.

\section{Simulation Result} \label{section:simulation}


We verify the tracking performance of the proposed control strategy in simulation. The desired trajectory is designed to highlight the effects of CoM shift: it consists of two piece-wise continuous and differentiable segments. The first segment highlights the multi-rotor's ability to follow a specified path under the controller. The manipulator deploys during the second path segment, thereby testing the system's ability to track the desired path with rapid change in overall CoM. In the simulation, an aerial manipulator with 2-link arm is modeled with the base vehicle mass being $2.7\, kg$, and the manipulator masses of the first and second path segments is $0.5\, kg$ and $1.0\, kg$, respectively, which mimics the delivery of a payload. The first and second link manipulator lengths are set to be $0.25\, m$ and $0.2\, m $. Moreover, the 2-link manipulator forms a planar manipulator with $\mathcal{R}_{L_2} = \mathcal{R}_{L_1} = I_{3\times 3}$, i.e. the appended arm is restricted to manipulate in $\hat{x}_b - \hat{z}_b$ plane. The control frequency is set to be 1 kHz, and a zero-order-hold (ZOH) is used between each controller update.
In Fig. \ref{fig:simresults}((a)), the controller provides promising tracking performance despite an initial disturbance. A comparison of control effort between the QP-based controller and flatness generated control action is given in Fig. \ref{fig:simresults}((b)). 
By deviating from the theoretical flatness control action as needed, the CLF-QP controller can reject initial state error and exponentially track the desired trajectory.

\section{Conclusion}\label{sec:Conclusion}
In summary, the EoM of a general class of broken symmetry aerial manipulators is derived using Lagrangian Reduction with advected parameters. Inspired by the dynamical coupling in the reduced equations, we theorize and prove that the outputs consisting of $\mathbold{p}_e$, overall CoM linear momentum, $\psi$, yaw position angles, and $\mathbold{\eta}$, manipulator relative joint angles, are differentially flat. The flat output parameterization of the control inputs necessitated a second-order dynamic extension of thrust input to allow a valid vector relative degree and dynamic feedback linearization. Using this extension, we introduced a CLF-QP-based exponentially stabilizing controller that guarantees local exponential tracking of the desired flat outputs.

In future work, we plan to demonstrate the controller's performance on a hardware system, integrated with flatness-based trajectory planning algorithms. Moreover, we seek to include the contact constraints that arise in perching-like behavior and physical contact and design analogous controllers that can accommodate such constraints.


\addtolength{\textheight}{-12cm}   

\bibliographystyle{ieeetr}
\bibliography{refs}

\end{document}